\documentclass{article}
\usepackage[eatrack]{log_2023}				

\usepackage{booktabs}						
\usepackage{multirow}						
\usepackage{amsfonts}						
\usepackage{graphicx}						

\usepackage[numbers,compress,sort]{natbib}	

\usepackage{url}            
\usepackage{nicefrac}       
\usepackage{microtype}      
\usepackage{xcolor}         
\usepackage{amsmath}
\usepackage{subfigure}
\usepackage{graphics}
\usepackage{mathtools}
\usepackage{amsthm}
\usepackage{amssymb}
\usepackage{tikz-cd}
\usepackage{wrapfig}
\usepackage{adjustbox}
\usepackage{tablefootnote}

\newtheorem{theorem}{Theorem}[section]

\newtheorem{corollary}[theorem]{Corollary}

\newcommand*{\Scale}[2][4]{\scalebox{#1}{$#2$}}%
\newcommand{\nopar}{{\parfillskip=0pt \par}}

\title[Propagate \& Distill: Towards Effective Graph Learners Using Propagation-Embracing MLPs]{Propagate \& Distill: Towards Effective Graph Learners Using Propagation-Embracing MLPs}

\author[Shin et al.]{%
Yong-Min Shin\\
Yonsei Univsersity\\
\email{jordan3414@yonsei.ac.kr}\And
Won-Yong Shin\thanks{Corresponding Author.}\\
Yonsei Univsersity\\
\email{wy.shin@yonsei.ac.kr}
}

\begin{document}

\maketitle
 
\begin{abstract}
Recent studies attempted to utilize multilayer perceptrons (MLPs) to solve semi-supervised node classification on graphs, by training a student MLP by knowledge distillation from a teacher graph neural network (GNN). While previous studies have focused mostly on training the student MLP by matching the output probability distributions between the teacher and student models during distillation, it has not been systematically studied how to inject the structural information in an \textit{explicit} and \textit{interpretable} manner. Inspired by GNNs that separate feature transformation $T$ and propagation $\Pi$, we re-frame the distillation process as making the student MLP learn both $T$ and $\Pi$. Although this can be achieved by applying the inverse propagation $\Pi^{-1}$ before distillation from the teacher, it still comes with a high computational cost from large matrix multiplications during training. To solve this problem, we propose {\bf Propagate \& Distill (P\&D)}, which propagates the output of the teacher before distillation, which can be interpreted as an approximate process of the inverse propagation. We demonstrate that {\bf P\&D} can readily improve the performance of the student MLP.
\end{abstract}

\section{Introduction}
\label{section:introduction}
Although message passing built upon the graph structure is crucial to the graph neural network (GNN)'s performance~\cite{xu2019gin, Morris2019WL}, it is known that this also causes a slow inference time~\citep{yan2020tinygnn, zhang2022glnn}, which sets a constraint on various real-world applications of GNNs, in particular where fast inference time is essential (\textit{e.g.,} web recommendation~\citep{hao2020pcompanion, zhang2020agl}). Very recently, GLNN~\citep{zhang2022glnn} proposed to replace GNNs with multiplayer perceptrons (MLPs), while training the student MLP with knowledge distillation (KD)~\citep{hinton2015kd} from a GNN teacher. By this approach, it was shown that the inference time has become over $\times$100 faster, while resulting in a satisfactory performance of the student MLP in semi-supervised node classification (SSNC). However, in such a scenario of GNN-to-MLP KDs, the student MLP is unable to utilize the structural information of graphs as input, resulting in a large information gap between the teacher GNN and the student MLP. In light of this, the main objective of GNN-to-MLP KDs is to enable the weights of the student MLP to \textit{learn} the graph structure so that, the student MLP achieves the prediction accuracy on par with the teacher GNN.

Several follow-up studies on GNN-to-MLP KDs since the success of GLNN~\cite{zhang2022glnn} mainly relied a common solution by leveraging the structural information as a part of the input to the student MLP~\citep{tian2023nosmog, chen2021samlp, zheng2022coldbrew} to achieve state-of-the-art performance, which however poses the following technical challenges.
\begin{itemize}
    \item \textbf{(Challenge 1)} In~\citep{tian2023nosmog}, the positional embeddings learned by DeepWalk~\cite{perozzi2014deepwalk} are concatenated into the node features as input to the student MLP. This necessitates consistent re-computation when the underlying graph evolves for maintenance.
    
    \item \textbf{(Challenge 2}) Other GNN-to-MLP approaches either directly utilize rows of the adjacency matrix, which makes the input dimension of the MLP dependent on the number of nodes~\citep{chen2021samlp}, or require a specific design of GNNs that is able to provide structural node features~\citep{zheng2022coldbrew}, thus reducing the flexibility of the model design.
\end{itemize}  
To tackle the above challenges, we aim to devise a new methodology such that MLPs are capable of learning the structural information while 1) {\it fixing} the input of the student MLP to the {\it node features only} and 2) injecting structural information during {\it training}. Additionally, since most GNN models stack up to only a few layers ({\it e.g.}, 2 GNN layers)~\citep{shchur2018pitfall, li2019deepgcn} in practice due to the oversmoothing problem~\citep{li2018deeper, chen2020measuring}, depending solely on the teacher GNN's output may disable the student MLP to capture high-order connectivity information, leaving a room for further performance improvement.

To achieve the goal of graph structure learning for MLPs along with the aforementioned constraints, we gain an insight from GNNs that separate feature transformation $T$ and propagation $\Pi$~\citep{gasteiger2019ppnpappnp, huang2021cns, bojchevski2020scalegnnappnp, chien2021generalappnp}. Based on these studies, we aim to further boost the performance of the student MLP using $\Pi$ during the KD process in an {\it explicit} and {\it interpretable} manner, which encodes a global structural view of the underlying graph. At its core, this approach begins by regarding the teacher GNN's output as a {\it base prediction} rather than the final prediction, allowing us to arrive at a formulation where the output of the student MLP first passes through an inverse propagation $\Pi^{-1}$ before being matched with the teacher GNN's output during KD. Although this approach can be interpreted as training the student MLP in such a way that it behaves as a graph learner embracing the propagation $\Pi$ by learning both $T$ and $\Pi$, it requires large matrix multiplications for each feed-forward process during the training process in KD. As a more efficient workaround, we propose {\bf Propagate \& Distill (P\&D)}, which approximates $\Pi^{-1}$ by recursively propagating the teacher GNN's output over the graph. Our approach also allows a room for more flexibility by adopting different propagation rules, akin to prior studies on label propagation~\citep{zhou2003lp, zhu2003lp2, huang2021cns}. We demonstrate the superiority of {\bf P\&D} on popular real-world benchmark datasets, and show that stronger propagation generally leads to better performance.

In summary, our contributions are as follows:
\begin{itemize}
    \item We present {\bf P\&D}, a simple yet effective GNN-to-MLP distillation method that allows additional structural information to be injected during training by recursively propagating the output of the teacher GNN;
    \item We empirically validate the effectiveness of {\bf P\&D} using real-world graph benchmark datasets for both transductive and inductive settings;
    \item We demonstrate that deeper and stronger propagation in {\bf P\&D} generally tends to achieve better performance.
\end{itemize}

\section{Methodology}
\label{section:proposedmethod}

In this section, we elaborate on our proposed framework {\bf P\&D}. We first describe the background of prior work that attempted to separate feature transformation and propagation. Then, we describe our formulation eventually leading to the proposed framework.


{\bf Background.} Typical GNN models stack multiple message passing layers, each of which consists of the propagation phase and the transformation phase~\citep{gilmer2017mpnn}. On the other hand, a handful of prior studies including~\citep{gasteiger2019ppnpappnp, huang2021cns, bojchevski2020scalegnnappnp, chien2021generalappnp} proposed to separate feature transformation and propagation in the GNN model. Given a feature vector ${\bf x}_i$, a feature transformation $\mathcal{T}: {\bf x}_i \rightarrow {\bf h}_i^t$ is first applied to calculate the base prediction ${\bf h}_i^t$, and then the GNN model propagates ${\bf h}_i^t$ along the underlying graph by a propagation operation $\Pi: {\bf h}_i^t \rightarrow {\bf p}_i^t$ to get the final prediction ${\bf p}_i^t$. As an example, PPNP~\citep{gasteiger2019ppnpappnp} employed an MLP model to learn the proper feature transformation $\mathcal{T}$ and utilized personalized PageRank (PPR) as the propagation operation $\Pi$. In PPNP, the propagation operation $\Pi = \Pi_{\text{PPR}}$ is characterized as $\Pi_{\text{PPR}} = (1 - \gamma)(I_{|\mathcal{V}|} -  \gamma\tilde{A})^{-1}$, where $1 - \gamma \in (0, 1]$ is the restart probability, $\mathcal{V}$ is the set of nodes, $I_{|\mathcal{V}|} \in \mathbb{R}^{|\mathcal{V}| \times |\mathcal{V}|}$ is an identity matrix, and $\tilde{A}$ is the symmetrically normalized adjacency matrix with self-loops. Such separation-based approaches have been shown to effectively encode the (global) structural information while avoiding the oversmoothing effect~\citep{gasteiger2019ppnpappnp, chien2021generalappnp}. In our study, to further boost the performance of the student MLP $f$, we re-frame the GNN-to-MLP KD problem by regarding the teacher GNN as $\text{GNN} \approx \Pi’ \circ T$, where $\Pi’$ does not perform enough propagation to benefit our graph learning task. In this new viewpoint, we do not want to solely rely on the output logits $P^t$ of the teacher GNN. Rather, $P^t \approx (\Pi’ \circ T)(X)$ can be further enhanced by an additional propagation $\Pi_{\text{PPR}}$ to complement $\Pi’$, and we set $f(X) = (\Pi_{\text{PPR}} \circ \Pi’ \circ T)(X)$ as our ideal objective of KD. However, since $\Pi=\Pi_{\text{PPR}}$ is defined as an (computationally expensive) inverse matrix, it is more reasonable to consider the distillation loss $\text{KL}((\Pi^{-1} \circ f)(X), (\Pi’ \circ T)(X))$, which results in:
\begin{equation}
    \mathcal{L}_\text{KL}=\text{KL}((\Pi^{-1} \circ f)(X), P^t) = \text{KL}((2 I_{|\mathcal{V}|} - \gamma \tilde{A})P^s, P^t),
    \label{eq:trainingobjective}
\end{equation}
where $\text{KL}$ denotes the Kullback–Leibler divergence, $P^s$ is the student MLP's output, and $\gamma$ comes from $\Pi_{\text{PPR}}$.\footnote{The constant term $(1-\gamma)^{-1}$ in $\Pi_{\text{PPR}}^{-1}$ can be ignored as we normalize both terms in the KL divergence loss.} In Eq.~(\ref{eq:trainingobjective}), we include an additional identity matrix, which can be interpreted as an additional skip-connection alongside $\Pi^{-1}$. By multiplying the term $2 I_{|\mathcal{V}|} - \gamma \tilde{A}$ by $P^s$ before calculating the loss, this formulation explicitly involves the structural information during training.

{\bf Propagate \& Distill (P\&D).} A downside of Eq.~(\ref{eq:trainingobjective}) is that, for every loss calculation during the KD process, we need to constantly multiply $2 I_{|\mathcal{V}|} - \gamma \tilde{A}$, thus increasing the computational cost. To remedy this, we choose a computationally efficient alternative formulation of Eq.~(\ref{eq:trainingobjective}) by \textit{approximating} the inverse matrix calculation in $\Pi_{\text{PPR}}$ along with a recursive formula, similarly as in the approach of APPNP~\citep{gasteiger2019ppnpappnp}. To this end, we propose {\bf P\&D}, which discovers an approximate propagation function $\Bar{\Pi} \approx \Pi$, where $\Bar{\Pi}$ is defined as a recursive formula that is applied to the output of the \textit{teacher GNN's prediction} instead of applying the inverse propagation function $\Pi^{-1}$ to the output of the student MLP. In other words, $\Bar{\Pi}$ propagates $P^t$ along the underlying graph by recursively applying
\begin{equation}
\label{equation:labelpropagation}
    P^{t}_{l+1} = \gamma \tilde{A} P^t_{l} + (1 - \gamma) P^t_{l},
\end{equation}
where we initially set $P^t_1 = P^t$ for $l = 1, \cdots, T$; and $\gamma \in (0,1]$ is a coefficient controlling the propagation strength through neighbors of each node. Denoting the propagation operation in {\bf P\&D} as $\Bar{\Pi}(P^t, \tilde{A})$, we now formulate our new distillation loss function $\mathcal{L}_{\textbf{P\&D}}$ as follows:
\begin{equation}
\label{equation:pndloss}
    \mathcal{L}_{\textbf{P\&D}} = \text{KL}(P^s, \Bar{\Pi}(P^t, \tilde{A})),
\end{equation}
which only requires an additional calculation of $\Bar{\Pi}(P^t, \tilde{A})$, leaving the rest of the KD process same as GLNN~\cite{zhang2022glnn}. Furthermore, this approach not only introduces another natural interpretation, but also allows a room for flexibility in the design of a recursive formula. Precisely, Eq.~(\ref{equation:labelpropagation}) can be seen as iteratively smoothing the output of the teacher's prediction along the graph structure, which is closely related to classic label propagation (LP) methods~\citep{zhou2003lp, zhu2003lp2} that propagate node label information rather than probability vectors. As in LP, {\bf P\&D} also takes advantage of the homophily assumption to potentially correct the predictions of incorrectly-predicted nodes with the aid of their (mostly correctly predicted) neighbors.\footnote{We refer to Appendix~\ref{section:SupplHomophily} for a theoretical analysis on the role of homophily.} Furthermore, thanks to the flexibility of the LP family, we introduce another variant, named as \textbf{P\&D}-fix. In this version, the $l$-th iteration of propagation now becomes
\begin{align}
    &\text{(Step 1) } P^{t}_{l+1} =  \gamma \tilde{A}P^t_{l} +(1 - \gamma) P^t_{l},\nonumber\\
    &\text{(Step 2) } P^{t}_{l+1}[j,:] = P^t[j, :] \text{ for } j \in \mathcal{V}_T, \label{equation:fixprop}
\end{align}
where $\mathcal{V}_T$ denotes the set of training nodes. Different from {\bf P\&D}, for every iteration, the output probability of training nodes gets manually replaced by the initial output probability (see Step 2 in Eq.~(\ref{equation:fixprop})). Adding Step 2 during propagation will lead to the initial output probability for some nodes in the training set as their predictions are expected to be nearly correct. In later descriptions, we denote {\bf P\&D} and {\bf P\&D}-fix as the versions using functions $\Bar{\Pi}$ and $\Bar{\Pi}_{\text{fix}}$, respectively. We also denote the previous inverse propagation approach in Eq.~(\ref{eq:trainingobjective}) as InvKD. Note that, during inference, we use $P^s$ as the student model's prediction, \textit{i.e.,} $f({\bf x}_i) = {\bf h}^s_i$.




\begin{table*}[!t]
\vskip -0.1in
\caption{Node classification accuracy (\%) for five different datasets in transductive and inductive settings. The columns represent the performance of the teacher GNN model, plain MLP model without distillation, GLNN~\citep{zhang2022glnn}, InvKD, and two versions of {\bf P\&D}. For each dataset, the performance of the best method is denoted in bold font.}
\label{table:mainexperiment}
\centering
\small
\resizebox{0.9\textwidth}{!}{%
\begin{tabular}{lcccccc}
\toprule
\textit{Transductive} & Teacher GNN & Plain MLP & GLNN & InvKD & {\bf P\&D} & {\bf P\&D}-fix  \\
\midrule
Cora & 78.81 ± 2.00 & 59.18 ± 1.60 & 80.73 ± 3.42 & 82.22 ± 1.45 & 82.16 ± 1.98 & \textbf{82.29} ± 1.60 \\
CiteSeer & 70.62 ± 2.24 & 58.51 ± 1.88 & 71.19 ± 1.36 &  74.08 ± 1.82 & 73.38 ± 1.39 & \textbf{74.93} ± 1.63 \\
Pubmed & 75.49 ± 2.25 & 68.39 ± 3.09 & 76.39 ± 2.36 & 77.22 ± 1.98 & 77.88 ± 2.89 & \textbf{78.11} ± 2.89 \\
A-Computer & 82.69 ± 1.26 & 67.79 ± 2.16 & 83.61 ± 1.49 & \textbf{83.81} ± 1.16 & 82.06 ± 1.58 & 83.21 ± 1.21 \\
A-Photo & 90.99 ± 1.34 & 77.29 ± 1.79 & 92.72 ± 1.11 &  92.83 ± 1.22 & 92.91 ± 1.31 & \textbf{93.02} ± 1.32 \\\midrule
\textit{Inductive} &  &  &  &  &  &  \\ \midrule
Cora & 80.61 ± 1.81 & 59.44 ± 3.36 & 73.07 ± 1.90 &  \textbf{75.18} ± 1.26 & 72.27 ± 2.74 & 71.24 ± 3.45 \\
CiteSeer & 69.83 ± 4.16 & 59.34 ± 4.61 & 68.37 ± 4.22 &  71.93 ± 3.16 & \textbf{72.87} ± 2.57 & 72.69 ± 2.40 \\
Pubmed & 75.25 ± 2.42 & 68.29 ± 3.26 & 75.01 ± 2.20 &  76.49 ± 2.47 & 76.49 ± 2.47 & \textbf{76.58} ± 2.34 \\
A-Computer & 83.06 ± 1.81 & 67.86 ± 2.16 & 79.77 ± 1.72 &  80.04 ± 1.90 & 80.28 ± 1.79 & \textbf{80.38} ± 1.59 \\
A-Photo & 91.21 ± 1.10 & 77.44 ± 1.50 & 89.73 ± 1.18 &  \textbf{90.28} ± 1.04 & 90.23 ± 1.02 & 89.87 ± 1.03 \\
\bottomrule
\end{tabular}
}
\vskip -0.2in
\end{table*}

\begin{table}
\centering
\small
\parbox{.49\textwidth}{
\centering
\caption{Node classification accuracy (\%) according to different $T$'s for the Cora dataset. The best performing cases are underlined.}
\Scale[0.94]{
\label{table:analysisTmini}
\begin{tabular}{lcccc}
\toprule
 $T$ & $\leq$5 & 10 & 20 & 50\\
\midrule
\multirow{2}{*}{{\em Trans.}} &  82.16      &   82.88          &  \underline{83.03} & 82.38 \\
                        &                    &  ($\uparrow$0.72) & ($\uparrow$0.87)     & ($\uparrow$0.22)\\\midrule
\multirow{2}{*}{{\em Ind.}}&  71.59           & \underline{72.27}   & 71.31      & 71.85\\
                        &                    & ($\uparrow$0.68)     & ($\downarrow$0.28)  & ($\uparrow$0.26)\\
\bottomrule
\end{tabular}
}
}
\hfill
\parbox{.49\textwidth}{
\centering
\caption{Node classification accuracy (\%) according to different $\gamma$'s for the Cora dataset. The performance gain of the case of $\gamma=0.9$ over the case of $\gamma=0.1$ is displayed in the parenthesis.}
\Scale[0.94]{
\label{table:analysisGmini}
\begin{tabular}{ccccc}
\toprule
\multicolumn{2}{c}{{\em Transductive}} & \multicolumn{2}{c}{{\em Inductive}}\\\midrule
$\gamma = 0.1$ & $\gamma = 0.9$ & $\gamma = 0.1$ & $\gamma = 0.9$\\
\midrule
80.35          & 82.16            & 70.87          & 71.99         \\
            & ($\uparrow$1.81) &                & ($\uparrow$1.12) \\
\bottomrule
\end{tabular}
}
}
\vskip -0.3in
\end{table}

\section{Main Results}
\label{section:maineresults}

In this section, we present experimental results to validate the effectiveness of {\bf P\&D} and {\bf P\&D}-fix with further empirical analyses of the approximate propagation function $\Bar{\Pi}$.\footnote{We refer to Appendix~\ref{appendix:analysisdetail} and~\ref{appendix:datasetstatistics} for further details.}

{\bf Experimental setup.} In our experiments, we mostly follow the settings of prior studies~\citep{zhang2022glnn, yang2021cpf}. Specifically, we focus only on the KL divergence loss (without the cross-entropy loss) for all experiments. We adopt a 2-layer GraphSAGE model~\citep{hamilton2017graphsage} with 128 hidden dimensions. We use the Adam optimizer~\citep{kingma2014adam}, batch size of 512, and early stopping with patience 50 during training. We report the average accuracy of the student MLP from 10 trials. 

{\bf Datasets.} We use the Cora, CiteSeer, Pubmed~\citep{sen2008ccp, yang2016ccp}, A-Computer, and A-Photo~\cite{shchur2018pitfall} datasets. We choose 20 / 30 nodes per class for the training / validation sets as in~\citep{shchur2018pitfall, zhang2022glnn}. For the inductive setting, we further sample 20\% of test nodes to be held out during training. 

{\bf Scenario settings.} In the transductive setting, we use the set of edges $\mathcal{E}$ and node features $X$ for all nodes in $\mathcal{V}$, along with the label information in the set of training nodes $\mathcal{V}_T$. In the inductive setting, all edges that connect nodes in the inductive subset $\mathcal{U}_{\text{ind}}$ and the rest of the graph (i.e., $\mathcal{V}_T \cup \mathcal{U}_{\text{obs}}$) are removed, and remain disconnected during the test phase, following~\citep{zhang2022glnn}.

{\bf Experimental results.} Table~\ref{table:mainexperiment} shows the performance comparison with GLNN~\citep{zhang2022glnn} as well as the teacher GNN and plain MLP models, in terms of the node classification accuracy for all five benchmark datasets. First, we observe that using one of InvKD, \textbf{P\&D}, and \textbf{P\&D}-fix consistently outperforms the benchmark methods regardless of datasets and scenario settings. For example, in the transductive setting using the CiteSeer dataset, {\bf P\&D}-fix exhibits the best performance with a gain of 3.74\% over GLNN, and in the inductive setting using the Cora dataset, InvKD is the best performer while showing a gain of 2.11\% over GLNN. Such a benefit from the inductive setting is meaningful since the propagation is performed without access to test nodes. 

We also investigate how the total number of iterations $T$ and the propagation strength $\gamma$ in Eq.~(\ref{equation:labelpropagation}) affects the performance. Here, we perform the analysis using {\bf P\&D} as our main framework.\footnote{We refer to Appendix~\ref{appendix:etcexperiments} for the analysis for other datasets and {\bf P\&D}-fix.} To see how the performance behaves with $T$, we consider four cases: $T \in \{1,2,5\}$, $T = 10$, $T = 20$, and $T = 50$. We measure the performance gain compared to the first case ({\it i.e.}, $T \in \{1,2,5\}$) using the Cora dataset. Table~\ref{table:analysisTmini} shows that, in both transductive/inductive settings, the best performance can be achieved when $T$ is sufficiently large ({\it i.e.}, $T\ge 10$). Next, to see how the performance behaves with $\gamma$, we consider two cases: $\gamma = 0.1$ and $\gamma = 0.9$. Table~\ref{table:analysisGmini} shows that stronger propagation ({\it i.e.}, $\gamma=0.9$) leads to higher performance in both cases.

\section{Conclusion}

We presented {\bf P\&D}, a simple yet effective method to boost the performance of MLP models trained by distillation from a teacher GNN model. We empirically showed that applying an approximate propagation $\Bar{\Pi}$ to the teacher GNN's output eventually benefits the student MLP model after KD on real-world graph benchmark datasets for both transductive and inductive settings. Our future work includes the potential enhancement of the inverse propagation in learning $\Pi^{-1}$ as a separate model.

\section*{Acknowledgements}
This work was supported by the National Research Foundation of Korea (NRF) grant funded by the Korea government (MSIT) (No. 2021R1A2C3004345, No. RS-2023-00220762).

\bibliographystyle{unsrtnat}
\bibliography{reference}

\appendix

\section{Preliminaries}
\label{section:background}
In this section, we summarize several preliminaries to our work, along with basic notations.

{\bf Semi-supervised node classification.} In SSNC, we are given a graph dataset $G = (\mathcal{V}, \mathcal{E}, X)$, where $\mathcal{V}$ is the set of nodes, $\mathcal{E} \subseteq \mathcal{V} \times \mathcal{V}$ is the set of edges, and $X \in \mathbb{R}^{|\mathcal{V}| \times d}$ is the node feature matrix where the $i$-th row ${\bf x}_i = X[i,:] \in \mathbb{R}^d$ is the $d$-dimensional feature vector of node $v_i \in \mathcal{V}$. We also denote $A \in \mathbb{R}^{|\mathcal{V}| \times |\mathcal{V}|}$ as the adjacency matrix to represent $\mathcal{E}$, where $A[i,j] = 1$ if $(i,j)\in\mathcal{E}$, and 0 elsewhere. The degree matrix $D = \text{diag}(A{\bf 1}_{|\mathcal{V}|})$ is a diagonal matrix whose diagonal entry $D[i, i]$ represents the number of neighbors for $v_i$, where ${\bf 1}_{|\mathcal{V}|}$ is the all-ones vector of dimension $|\mathcal{V}|$. Alongside $G$, $\mathcal{Y}$ indicates the set of class labels, and each node $v_i$ is associated with a ground truth label $y_i \in \mathcal{Y}$. Typically, $y_i$ is encoded as a one-hot vector ${\bf y}_i \in \mathbb{R}^{|\mathcal{Y}|}$. In SSNC, we assume that only a small subset of nodes $\mathcal{V}_T \subset \mathcal{V}$ have their class labels known during training. The objective of SSNC is to predict the class label for the rest of the nodes in $\mathcal{U} = \mathcal{V} \setminus \mathcal{V}_T$. In the transductive setting, we assume that access to information other than its labels (i.e., node features and associated edges) for all nodes in $\mathcal{U}$ is available during training. In the inductive setting, a held-out subset of nodes in $\mathcal{U}$ by separating into two disjoint subsets, namely the observed subset $\mathcal{U}_{\text{obs}}$ and the inductive subset $\mathcal{U}_{\text{ind}}$ ($\mathcal{U} = \mathcal{U}_{\text{obs}} \cup \mathcal{U}_{\text{ind}}$). The inductive subset $\mathcal{U}_{\text{ind}}$ is completely unknown during training, and the objective is to predict the labels of those unseen nodes.

{\bf KD from GNNs to MLPs.} Recently, several studies have put their efforts to leverage MLP models as the main architecture for SSNC~\citep{zhang2022glnn, tian2023nosmog, hu2021graphmlp, zheng2022coldbrew, chen2021samlp, Dong2022mlpgraphlearner, wu2023pgkd, Wu2023mlpgraphlearner}. Most of these attempts adopt the KD framework~\citep{bucila2006modelcompression, hinton2015kd, gou2021kdsurvey} by transferring knowledge from a teacher GNN to a student MLP. As the core component, knowledge transfer is carried out by matching soft labels via a loss function $\mathcal{L}_{\text{KL}}$, which plays a role of matching the output probability distributions between the teacher and student models with respect to the Kullback–Leibler (KL) divergence. Although other distillation designs have been proposed since~\citep{hinton2015kd}, distillation via $\mathcal{L}_{\text{KL}}$ has been a popular choice and is adopted in lots of follow-up studies that aim to distill knowledge from GNNs to MLPs.

More precisely, in our distillation setting where an MLP is trained via distillation from a teacher GNN, we first assume that the output of the teacher GNN, $H^t \in \mathbb{R}^{|\mathcal{V}| \times |\mathcal{Y}|}$, is given, where ${\bf h}^t_i = H^t[i,:]$ represents the output logit for node $v_i$. The objective of KD from GNNs to MLPs is to train the student MLP model $f$, which returns an output logit $f({\bf x}_i) = {\bf h}^s_i$ for a given node feature vector ${\bf x}_i$ of node $v_i$ as input. The two output logits ${\bf h}_i^t$ and ${\bf h}^s_i$ are transformed into class probability distributions by a softmax function, i.e., ${\bf p}^t_i = \texttt{softmax}({\bf h}^t_i)$ and ${\bf p}^s_i = \texttt{softmax}({\bf h}^s_i)$, respectively. During distillation, $\mathcal{L}_{\text{KL}} \triangleq \text{KL}({\bf p}^s_i, {\bf p}^t_i)$ compares the student's output probability ${\bf p}_i^s$ and the teacher's output probability ${\bf p}_i^t$ by a KL divergence loss. A mix of $\text{KL}({\bf p}^s_i, {\bf p}^t_i)$ and the cross-entropy loss, denoted as $\text{CE}({\bf p}^s_i, {\bf y}_i)$, with labeled nodes is used as a final loss function:
\begin{equation}
\label{eq:distillloss}
    \mathcal{L}_{\text{Distill}} = \alpha \sum_{i \in \mathcal{V}_T} \text{CE}({\bf p}^s_i, {\bf y}_i) + (1 - \alpha) \sum_{i \in \mathcal{V}} \text{KL}({\bf p}^s_i, {\bf p}^t_i),
\end{equation}
where $\alpha \in [0,1]$ is a mixing parameter. After training, only the MLP model $f$ is used during inference, which dramatically improve the computational efficiency since the feed-forward process basically involves only matrix multiplication and element-wise operations, without message passing~\citep{zhang2022glnn, tian2023nosmog}. Since a majority of GNN-to-MLP distillation methods adopt only the second term as their distillation loss~\citep{zhang2022glnn, tian2023nosmog, chen2021samlp}, we also focus on Eq.~(\ref{eq:distillloss}) to set $\alpha=0$ in the KL divergence loss in our study.

\section{Case Study on Interpretations} 
\label{section:CaseStudy}

\begin{figure}[h]
  \includegraphics[width=\textwidth]{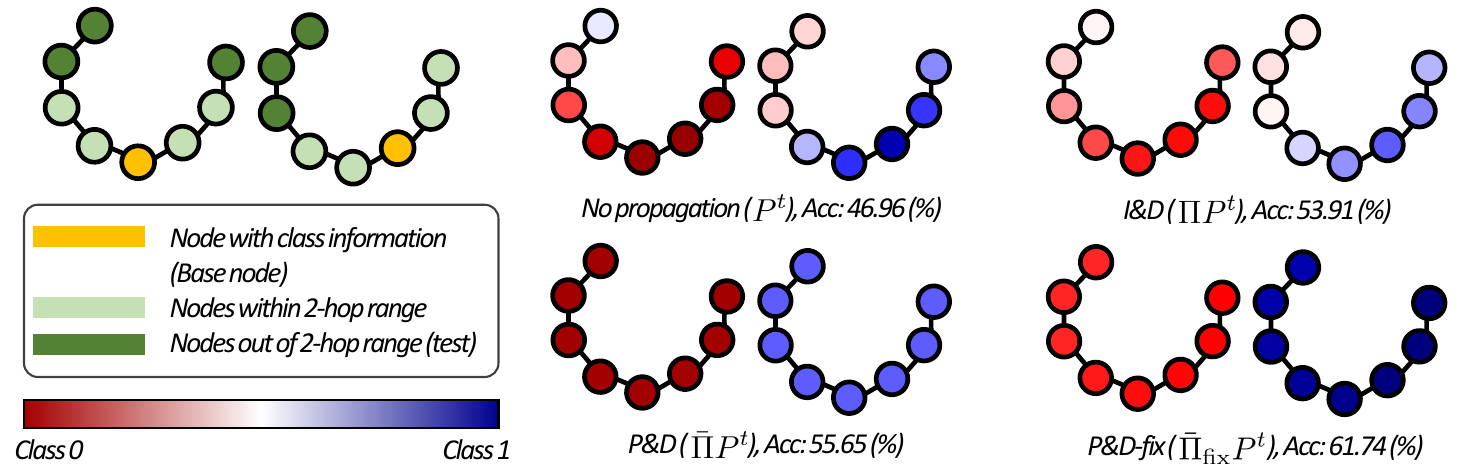}
  \caption{Visualization of the effect of various propagation functions in InvKD and {\bf P\&D} on the synthetic Chains dataset.}\label{fig:chainsexperiment}
\end{figure}

In order to provide interpretations on the benefits of injecting additional structure information in InvKD and \textbf{P\&D} during distillation, we perform experiments on the synthetic Chains dataset~\cite{gu2020ignn}, which consists of 30 chain graphs of a fixed length of 8. Fig.~\ref{fig:chainsexperiment} visualizes 2 chains for each propagation for ease of presentation. All nodes in the same chain are assigned to the same class, and the class information is provided as a one-hot representation in the feature vector of \textit{one} of the nodes (the base node) in the chain. To train the teacher GNN, we adopt a 2-layer GraphSAGE model, which is thus able to exploit connectivities only within 2-hop neighbors of the base node. Then, we plot $P^t$, $\Pi P^t$, $\Bar{\Pi} P^t$, and $\Bar{\Pi}_{\text{fix}}P^t$, which correspond to the case for GLNN, InvKD, \textbf{P\&D}, \textbf{P\&D}-fix, respectively.\footnote{Although we do not directly use $\Pi P^t$ in the loss in InvKD, we can consider this as the final prediction when the student MLP ideally achieves the zero loss during training.} Compared to $P^t$ where the teacher GNN only correctly predicts the nodes near base node, other propagations $\Pi P^t$, $\bar{\Pi} P^t$, and $\bar{\Pi}_\text{fix} P^t$ further spread the correct label information along the graph, while self-correcting the base prediction of $P^t$. Additionally, we evaluate the accuracy of the student MLP on the nodes further than 2-hops away from the base nodes (see dark green nodes in the left part of Fig.~\ref{fig:chainsexperiment}). We can observe that self-correction indeed benefits the student MLP. For example, for the nodes out of the 2-hop range, using $\Bar{\Pi}_{\text{fix}}P^t$ results in the accuracy of $61.74\%$ compared to the case of using $P^t$ showing the accuracy of $46.96\%$. This case study clearly validates the effect of our inverse/recursive propagation functions.\nopar

\section{Further Results of Propagation Analysis}
\label{appendix:etcexperiments}

In this subsection, we provide additional results of the effects of $T$ and $\gamma$ for both \textbf{P\&D} and \textbf{P\&D}-fix. Table~\ref{table:analysisTfull} and~\ref{table:analysisGfull} show the effects of $T$ and $\gamma$, respectively, for \textbf{P\&D} and \textbf{P\&D}-fix. Table~\ref{table:analysisTfull} shows that indicating that higher values of $T$ are more beneficial holds for other datasets on both \textbf{P\&D} and \textbf{P\&D}-fix, with the exception of \textbf{P\&D} on Citeseer for the inductive setting. Also, Table~\ref{table:analysisGfull} indicates that higher values of $\gamma$ are more beneficial holds for other datasets on both \textbf{P\&D} and \textbf{P\&D}-fix.

\begin{table}[h]
\centering
\caption{Node classification accuracy (\%) according to different $T$'s for three different datasets. The best performing cases are underlined.}
\begin{small}
\Scale[0.9]{
\label{table:analysisTfull}
\begin{tabular}{lccccccc}
\toprule
 &  & \multicolumn{3}{c}{\textbf{P\&D}} & \multicolumn{3}{c}{\textbf{P\&D}-fix} \\
\cmidrule{3-8}
 $T$ & & Cora & CiteSeer & PubMed & Cora & CiteSeer & PubMed \\
\midrule
\multirow{2}{*}{$\leq$5}& {\em Trans.} & 82.16 & 73.72 & 76.68 & 81.64 & 74.93 & 77.14\\
                        & {\em Ind.} & 71.59 & \underline{72.87} & 76.49 & 71.24 & 72.69 & 76.39\\
\midrule
\multirow{4}{*}{10}     & \multirow{2}{*}{{\em Trans.}} & 82.88 & 73.65 & \underline{77.88} & \underline{82.35} & 74.01 & 77.06 \\
                        & & ($\uparrow$0.72)  & ($\downarrow$0.07) & ($\uparrow$1.20) & ($\uparrow$0.71)  & ($\downarrow$0.92) & ($\downarrow$0.08)\\
\cmidrule{2-8}
                        & \multirow{2}{*}{{\em Ind.}} & \underline{72.27} & 72.21 & 76.57 & 70.59 & 70.59 & \underline{76.59}\\
                        & & ($\uparrow$0.68)  & ($\downarrow$0.66) & ($\uparrow$0.08) & ($\downarrow$0.65)  & ($\downarrow$2.10) & ($\uparrow$0.20)\\
\midrule
\multirow{4}{*}{20}    & \multirow{2}{*}{{\em Trans.}} & \underline{83.03} & \underline{73.74} & 76.56 & 81.85 & 74.04 & 77.54\\
                       & & ($\uparrow$0.87)  & ($\uparrow$0.02) & ($\downarrow$0.12) & ($\uparrow$0.21)  & ($\downarrow$0.89) & ($\uparrow$0.40)\\
\cmidrule{2-8}
                        & \multirow{2}{*}{{\em Ind.}} & 71.31 & 72.21 & \underline{76.62} & \underline{71.85} & 71.85 & 76.36\\
                        & & ($\downarrow$0.28)  & ($\downarrow$0.66) & ($\uparrow$0.13) & ($\uparrow$0.61)  & ($\downarrow$0.84) & ($\downarrow$0.03)\\
\midrule
\multirow{4}{*}{50}     & \multirow{2}{*}{{\em Trans.}} & 82.38 & 73.38 & 77.01 & 82.29 & \underline{74.97} & \underline{78.11}\\
                        & & ($\uparrow$0.22)  & ($\downarrow$0.34) & ($\uparrow$0.33) & ($\uparrow$0.65)  & ($\uparrow$0.04) & ($\uparrow$0.97)\\
\cmidrule{2-8}
                        & \multirow{2}{*}{{\em Ind.}} & 71.85 & 71.77 & 76.48 & 71.66 & \underline{72.76} & 76.58\\
                        & & ($\uparrow$0.26)  & ($\downarrow$1.10) & ($\downarrow$0.01) & ($\uparrow$0.42)  & ($\uparrow$0.07) & ($\uparrow$0.19)\\
\bottomrule
\end{tabular}
}
\end{small}
\end{table}

\begin{table}[h]
\caption{Node classification accuracy (\%) according to different $\gamma$'s for three different datasets for \textbf{P\&D} and \textbf{P\&D}-fix. The performance gain of the case of $\gamma = 0.9$ over the case of $\gamma = 0.1$ is displayed in the parenthesis.}
\label{table:analysisGfull}
\begin{small}
\begin{center}
\Scale[0.9]{
\begin{tabular}{lcccccccc}
\toprule
    & \multicolumn{4}{c}{\textbf{P\&D}} & \multicolumn{4}{c}{\textbf{P\&D}-fix} \\
\cmidrule{2-9}
\multirow{2}{*}{Dataset}& \multicolumn{2}{c}{{\em Transductive}} & \multicolumn{2}{c}{{\em Inductive}} & \multicolumn{2}{c}{{\em Transductive}} & \multicolumn{2}{c}{{\em Inductive}}\\
\cmidrule{2-9}
                        & $\gamma = 0.1$ & $\gamma = 0.9$ & $\gamma = 0.1$ & $\gamma = 0.9$ & $\gamma = 0.1$ & $\gamma = 0.9$ & $\gamma = 0.1$ & $\gamma = 0.9$\\
\midrule
\multirow{2}{*}{Cora}    & 80.35 & 82.16 & 70.87 & 71.99 & 80.85 & 82.35 & 69.93 & 71.24\\
                         & & ($\uparrow$1.81) & & ($\uparrow$1.12) & & ($\uparrow$1.50) & & ($\uparrow$1.31) \\
\midrule
\multirow{2}{*}{CiteSeer}    & 72.70 & 73.38 & 71.60 & 72.87 & 74.47 & 74.93 & 69.93 & 72.69\\
                         & & ($\uparrow$0.68) & & ($\uparrow$1.27) & & ($\uparrow$0.46) & & ($\uparrow$2.76) \\
\midrule
\multirow{2}{*}{PubMed}    & 76.56 & 77.88 & 75.84 & 76.49 & 76.89 & 78.11 & 75.79 & 76.58\\
                         & & ($\uparrow$1.32) & & ($\uparrow$0.65) & & ($\uparrow$1.22) & & ($\uparrow$0.79) \\
\bottomrule
\end{tabular}
}
\end{center}
\end{small}
\end{table}

\section{Connections to Graph Signal Denoising} 
We can interpret the basis of our framework as applying graph signal denoising (GSD) to the student MLP. We first present the following theorem that connects propagation and GSD:

\begin{theorem}[GSD of PPNP~\citep{ma2021gsd}]
\label{theorem:originalgraphsignaldenoising}
    Given a noisy signal $S \in \mathbb{R}^{|\mathcal{V}| \times |\mathcal{Y}|}$, PPNP~\citep{gasteiger2019ppnpappnp} solves a GSD problem, where the goal is to recover a clean signal $F \in \mathbb{R}^{|\mathcal{V}| \times |\mathcal{Y}|}$ by solving the following optimization problem: 
    \begin{equation}
    \label{ep:graphsignaldenoisingobjective}
        \text{arg}\min_{F} \mathcal{L}_{\text{GSD}} = ||F - S||^2 + (1/(1 -\gamma) - 1) \text{tr}(F^\top L F),
    \end{equation}
    where $L = I_{|\mathcal{V}|} - \tilde{A}$ is the Laplacian matrix.
\end{theorem}

\begin{figure}[h]
\vskip -0.2in
    \centering
    \subfigure[Cora.]{
    \centering
    \includegraphics[width=0.3\textwidth]{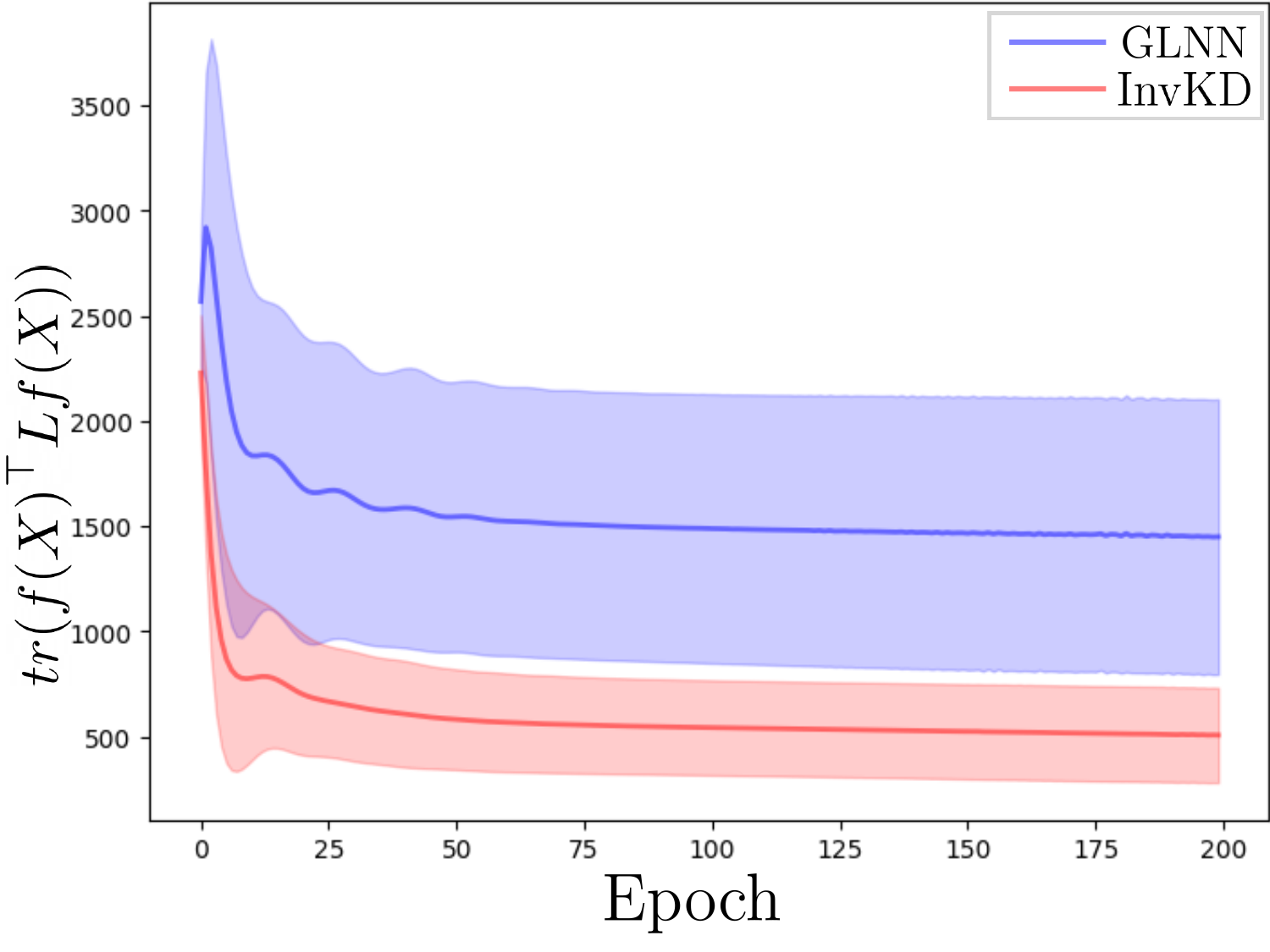}
    \label{figure:app-lap-a}
    }
    \centering
    \subfigure[CiteSeer.]{
    \centering
    \includegraphics[width=0.3\textwidth]{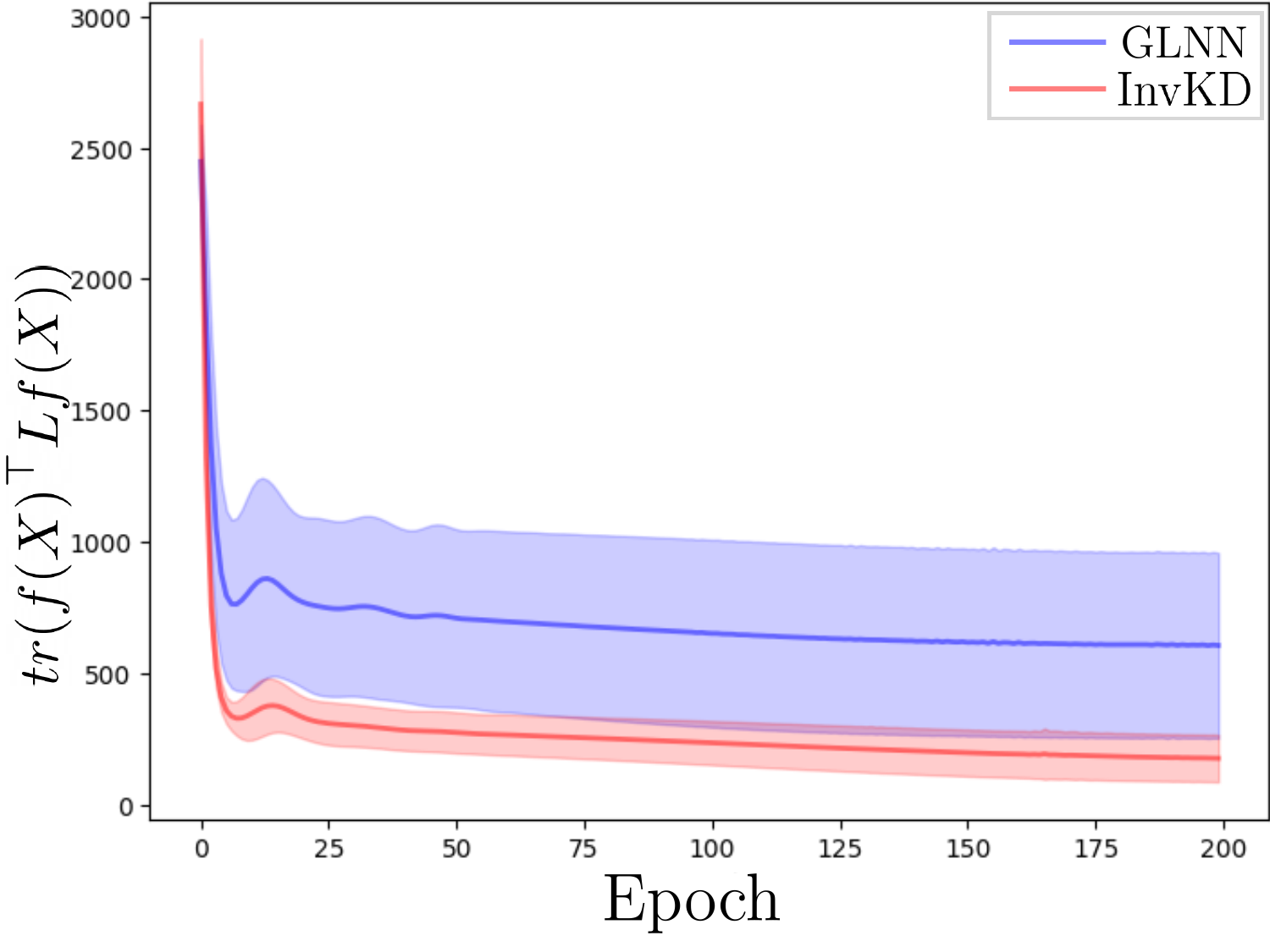}
    \label{figure:app-lap-b}
    }
    \centering
    \subfigure[PubMed.]{
    \centering
    \includegraphics[width=0.3\textwidth]{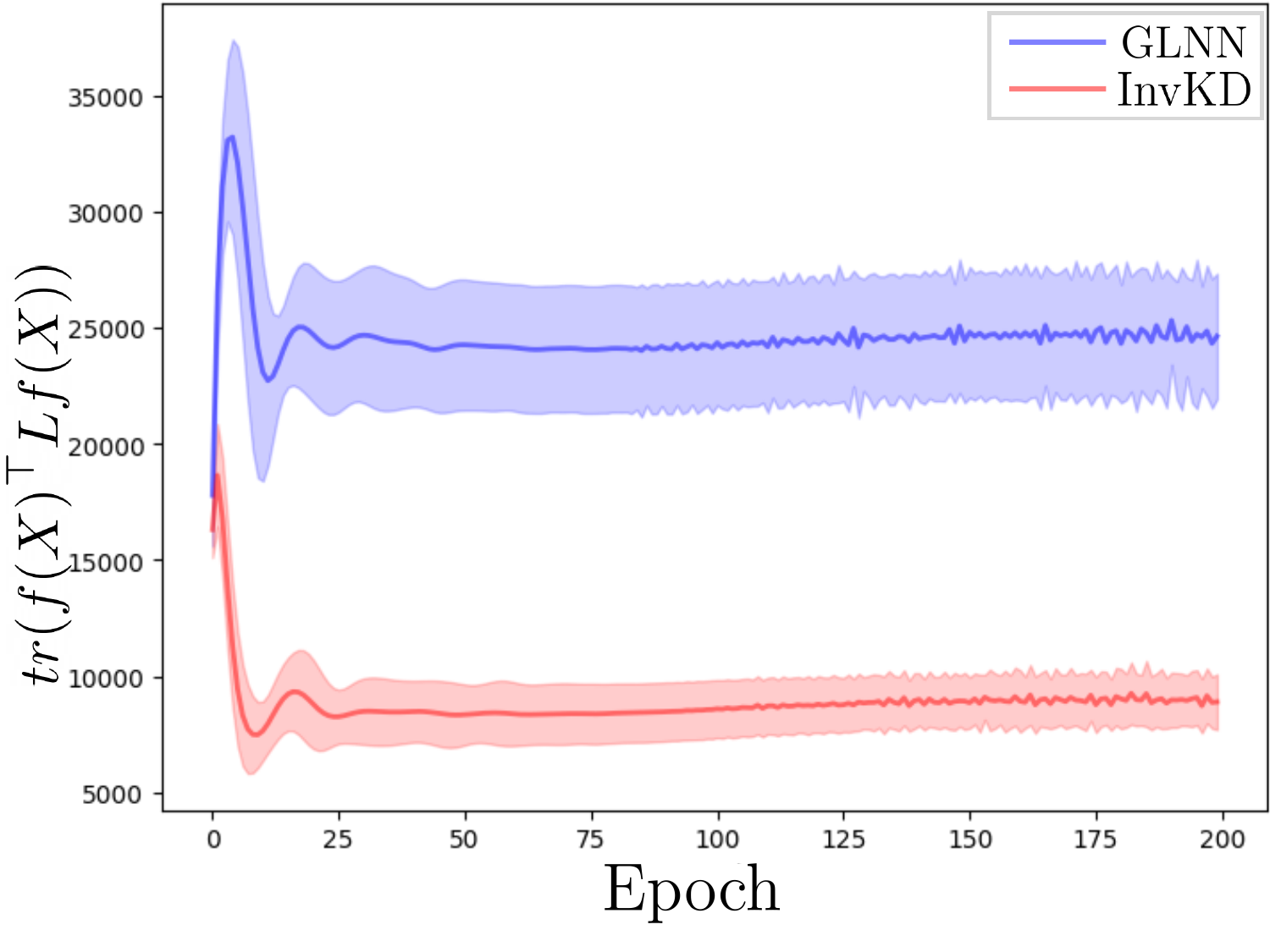}
    \label{figure:app-lap-c}
    }
    \caption{$\text{tr}(f(X)^\top L f(X))$ versus the number of epochs.}
    \label{fig:laplacianvsepoch}
\end{figure}

By interpreting $F = f(X)$ and $S = P^t$ in Theorem~\ref{theorem:originalgraphsignaldenoising}, InvKD can be seen as training the student MLP to fit a given signal $P^t$ with an additional smoothing constraint $(1/(1 -\gamma) - 1) \text{tr}(F^\top L F)$. To explicitly see this effect, we perform a experiment on the Cora, Citeseer, and Pubmed datasets where we plot the regularization term $\text{tr}(f(X)^\top L f(X))$ versus the number of epochs during training for GLNN and InvKD. 

The results in Fig.~\ref{fig:laplacianvsepoch} shows that a much stronger regularization effect for the case of InvKD (red) compared to na\"ive KD (blue). In conclusion, the inverse propagation $\Pi^{-1}$ in InvKD further forces the student MLP to return a signal smoothened over the graph.

\section{Ablation Study on the Inverse Propagation}


During our development in InvKD where {\bf P\&D} is based on, our inverse propagation function $\Pi^{-1}$ reveals a form similar to the Laplacian matrix $I_{|\mathcal{V}|}-\tilde{A}$ (see Eq.~(\ref{eq:trainingobjective})). Then, a natural question raising is ``Is $\Pi^{-1}$ in InvKD replaceable with alternative operations during distillation?". To answer this question, we run a simple experiment by taking into account two alternatives. First, one can expect that the student model $f$ may also learn the structural information when a convolution function ({\it i.e.}, the adjacency matrix with self-loops) is applied to the output of the student MLP instead of $\Pi^{-1}$ since the gradients of the model parameters are also influenced by $\tilde{A}$. Thus, we consider training the student MLP with $\mathcal{L}_{\text{conv}} \triangleq \text{KL}(\tilde{A}P^s, P^t)$. Second, as another alternative, we use $\mathcal{L}_{\text{Distill}}$ in Eq.~(\ref{eq:distillloss}), which does not contain any additional operation on $P^s$. We follow the same model configurations as those in Section~\ref{section:maineresults}, while using the node splits of~\citep{yang2016ccp} with full-batch training in the transductive setting. 

\begin{table}[h]
\caption{Performance comparison when three different loss functions are used during training in the transductive setting.}
\label{table:indanalysis}
\small
\centering
\begin{tabular}{lccc}
\toprule
Dataset  & $\mathcal{L}_{\text{conv}}$ & $\mathcal{L}_{\text{Distill}}$ & $\mathcal{L}_{InvKD}$  \\
\midrule
Cora & 64.78 ± 0.92 & 73.84 ± 0.52 & 76.20 ± 0.48 \\
CiteSeer & 69.71 ± 0.48 & 70.23 ± 0.40 & 72.43 ± 0.56 \\
PubMed & 60.27 ± 9.64 & 76.77 ± 1.02  & 77.67 ± 0.73 \\
\bottomrule
\end{tabular}
\end{table}

Table~\ref{table:indanalysis} shows the experimental results for three different cases when three datasets including Cora, Citeseer, and PubMed are used in the transductive setting. Interestingly, we can observe that the case of using $\mathcal{L}_{\text{conv}}$ exhibits much lower performance than that of other two cases.
This is because, while using $\mathcal{L}_{\text{conv}}$ explicitly accommodates the graph structure during distillation, it rather alleviates the pressure for the student MLP to learn the graph structure as $\tilde{A}$ is multiplied anyway during training; however, during inference, the student MLP have lower information on the graph structure with $\tilde{A}$ gone, which eventually harms the performance. Hence, this implies that using a na\"ive alternative operation may deteriorate the performance and a judicious design of the propagation operation is essential in guaranteeing state-of-the-art performance.

\section{Statistics of Datasets}
\label{appendix:datasetstatistics}
We show the table that contains the statistics of the six real-world datasets used in the experiments in Table~\ref{table:datasetstats}.

\begin{table}[h]
\caption{Statistics of six real-world datasets. NN, NE, NF, and NC denote the number of nodes, the number
of edges, the number of node features, the number of classes, respectively.}
\label{table:datasetstats}
\vskip 0.15in
\begin{center}
\begin{small}
\begin{tabular}{lcccc}
\toprule
Dataset & NN & NE  & NF & NC \\
\midrule
Cora & 2,485 & 5,069 & 1,433 & 7\\
CiteSeer & 2,120 & 3,679 & 3,703 & 6\\
Pubmed & 19,717 & 44,324 & 500 & 3\\
A-computer & 13,381 & 245,778 & 767 & 10\\
A-photo & 7,487 & 119,043 & 745 & 8\\
\bottomrule
\end{tabular}
\end{small}
\end{center}
\vskip -0.1in
\end{table}

\section{Further Experiment Details}
\label{appendix:analysisdetail}
\subsection{Model Hyperparameters}
The hyperparameters for the GraphSAGE model used as the teacher GNN including the number of layers, hidden dimension, learning rate, {\it etc.} are configured by essentially following the values of prior studies~\citep{yang2021cpf, zhang2022glnn, tian2023nosmog}. The full details of the settings are shown in Table~\ref{table:hyperparameter}.

\begin{table}[h]
\caption{Hyperparameter settings for the teacher GNN used in the experiments.}
\label{table:hyperparameter}
\vskip 0.15in
\begin{center}
\begin{small}
\begin{tabular}{lcccccc}
\toprule
 & Num. of layers & Hidden dim. & Learning rate & Dropout ratio & Weight decay & Fan out \\
\midrule
Setting & 2 & 128 & 0.01 & 0 & 0.0005 & 5,5 \\
\bottomrule
\end{tabular}
\end{small}
\end{center}
\vskip -0.1in
\end{table}

For the student MLP in P\&D, we perform a hyperparameter search of the learning rate in [0.01, 0.005, 0.001], weight decay in [0.005, 0.001, 0], dropout ratio in [0, 0.1, 0.2, 0.3, 0.4, 0.5, 0.6, 0.7, 0.8], and $\gamma$ in [0.1, 0.2, 0.3, 0.4, 0.5, 0.6, 0.7, 0.8, 0.9]. For P\&D and P\&D-fix, we perform the same hyperparameter search for learning rate, weight decay, and dropout ratio as in P\&D, while performing the search of $\gamma$ in [0.1, 0.9] and $T$ in [1, 2, 5, 10, 20, 50].

\subsection{Implementation Details}

For implementing our experiments, we use Pytorch version 1.12.0, and all GNN models are implemented using Pytorch Geometric version 2.0.4 and DGL version 0.9.1.

\section{Theoretical Analysis for Self-Correction via Propagation}
\label{section:theoreticalanalysis}
In this section, we provide a theoretical analysis of propagating the teacher's output $P^t$ along the graph in the context of {\it self-correction}. {\bf P\&D} relies on the assumption that the underlying graph has high levels of homophily, which measures the ratio of edges where the two connected nodes have the same class label~\citep{zhu2020h2gcn}. In this setting, we are interested in analyzing the condition where the prediction of a (incorrectly-predicted) node becomes corrected after one iteration of propagation by Eq.~(\ref{equation:labelpropagation}). We start by formally addressing basic settings and assumptions, which essentially follow those of~\citep{zhu2020h2gcn}. Let us assume that the underlying graph $G$ is regular ({\it i.e.}, all nodes have a degree of $d$) and $h\in[0,1]$ portion of neighbors have the same label for all nodes in $v\in\mathcal{V}$. For each node, the teacher GNN is assumed to be assigned an output probability vector having a probability $p \in [0,1]$ (with $p > 1/|\mathcal{Y}|$) for the true class label and another probability $(1-p)/(|\mathcal{Y}|-1)$ for the rest of the classes if the teacher GNN always makes predictions correctly.

Now, without loss of generality, let us assume that the teacher GNN makes an {\em incorrect} prediction for a particular node of interest $v_{*}$ with class 0 as its ground truth label by assigning a new output probability vector
\begin{equation}
    P^t[*,:] = \left[q, \dfrac{1-q}{|\mathcal{Y}|-1}, \cdots, \dfrac{1-q}{|\mathcal{Y}|-1}\right]
\end{equation}
with $0 < q < 1/|\mathcal{Y}|$, thus no longer assigning class 0 as its prediction. Additionally, for the \textit{rest} of the nodes, we introduce an error ratio $\epsilon \in (0, 1)$ to the teacher GNN's predictions, which assumes that the teacher model provides incorrect predictions for $\epsilon(|\mathcal{V}|-1)$ nodes (excluding $v_{*}$ itself). For the sake of simplicity, we assume that the probability of a node being incorrect by the teacher GNN is independent of its ground truth class label, which establishes the following theorem:

\begin{theorem}
\label{theorem:smoothingcorrectionwitherror}
    Suppose that the teacher GNN provides incorrect predictions for $\epsilon(|\mathcal{V}| - 1)$ nodes other than node $v_*$ where $\epsilon\in(0,1)$. Then, using one iteration of propagation in Eq.~(\ref{equation:labelpropagation}), the prediction of node $v_*$ gets corrected if
    \begin{equation}
    \label{equation:smoothingcorrectionwitherror}
        q \in \left[\max\left(0, \dfrac{1}{|\mathcal{Y}|} - \dfrac{\gamma}{1 - \gamma}\left( C - b(\epsilon)\right)\right), \dfrac{1}{|\mathcal{Y}|}\right],
    \end{equation}
    where $q$ is the output probability of $v_*$ corresponding to class 0, $C$ is approximately $\left(1 + \dfrac{1}{|\mathcal{Y}|}\right)hp - \dfrac{h+p}{|\mathcal{Y}|}$, $\gamma$ is the propagation strength in Eq.~(\ref{equation:labelpropagation}), and $b(\epsilon) = \left(C + \dfrac{hp}{|\mathcal{Y}|}\right)\epsilon$.
\end{theorem}
\begin{proof} 
Before proving Theorem~\ref{theorem:smoothingcorrectionwitherror}, we first show preliminary calculations that are needed before we proceed with the main proof. Here, we start with the simplest case. Following~\citep{zhu2020h2gcn}, we calculate the result when one-hot labels are being propagated, which will be used for later calculations. First, without loss of generality, we reorder the label matrix $Y \in \{0,1\}^{|V| \times |\mathcal{Y}|}$ (and also the rows and columns of adjacency matrix $A$) as follows:
\begin{equation}
\Scale[0.8]{
    Y =
    \begin{bmatrix}
        1 & 0 & \cdots & 0 \\
        \vdots & \vdots & \ddots & \vdots \\
        1 & 0 & \cdots & 0 \\
        0 & 1 & \cdots & 0 \\
        \vdots & \vdots & \ddots & \vdots \\
        0 & 1 & \cdots & 0 \\
        \vdots & \vdots &  & \vdots \\
        0 & 0 & \cdots & 1 \\
        \vdots & \vdots & \ddots & \vdots \\
        0 & 0 & \cdots & 1 \\
    \end{bmatrix}.
}
\end{equation}
We then calculate $(A + I)Y$, which we will progressively modify to $\Bar{\Pi}$ during the rest of this section. We take advantage of the neighbor assumptions, which results in:
\begin{equation}
\label{eq:naivelp}
\Scale[0.8]{
    (A + I)Y =
    \begin{bmatrix}
        hd + 1 &  \dfrac{1-h}{|\mathcal{Y}| - 1}d & \cdots & \dfrac{1-h}{|\mathcal{Y}| - 1}d \\
        \vdots & \vdots & \ddots & \vdots \\
        hd + 1 & \dfrac{1-h}{|\mathcal{Y}| - 1}d & \cdots & \dfrac{1-h}{|\mathcal{Y}| - 1}d \\
        \dfrac{1-h}{|\mathcal{Y}| - 1}d & hd + 1 & \cdots & \dfrac{1-h}{|\mathcal{Y}| - 1}d \\
        \vdots & \vdots & \ddots & \vdots \\
        \dfrac{1-h}{|\mathcal{Y}| - 1}d & hd + 1 & \cdots & \dfrac{1-h}{|\mathcal{Y}| - 1}d \\
        \vdots & \vdots &  & \vdots \\
        \dfrac{1-h}{|\mathcal{Y}| - 1}d & \dfrac{1-h}{|\mathcal{Y}| - 1}d & \cdots & hd + 1 \\
        \vdots & \vdots & \ddots & \vdots \\
        \dfrac{1-h}{|\mathcal{Y}| - 1}d & \dfrac{1-h}{|\mathcal{Y}| - 1}d & \cdots & hd + 1 \\
    \end{bmatrix}.
}
\end{equation}

We now start to modify this result of calculating $(A + I)Y$ that more resembles Eq.~(\ref{equation:labelpropagation}), except that we are still propagating one-hot labels. First, we introduce $\gamma$ in Eq.~(\ref{eq:naivelp}) and calculate $(\gamma A + (1 - \gamma)I)Y$: 
\begin{equation}
\Scale[0.8]{
    (\gamma A + (1 - \gamma)I)Y =
    \begin{bmatrix}
        \gamma hd + (1 - \gamma) &  \gamma\dfrac{1-h}{|\mathcal{Y}| - 1}d & \cdots & \gamma\dfrac{1-h}{|\mathcal{Y}| - 1}d \\
        \vdots & \vdots & \ddots & \vdots \\
        \gamma hd + (1 - \gamma) & \gamma\dfrac{1-h}{|\mathcal{Y}| - 1}d & \cdots & \gamma\dfrac{1-h}{|\mathcal{Y}| - 1}d \\
        \gamma\dfrac{1-h}{|\mathcal{Y}| - 1}d & \gamma hd + (1 - \gamma) & \cdots & \gamma\dfrac{1-h}{|\mathcal{Y}| - 1}d \\
        \vdots & \vdots & \ddots & \vdots \\
        \gamma\dfrac{1-h}{|\mathcal{Y}| - 1}d & \gamma hd + (1 - \gamma) & \cdots & \gamma \dfrac{1-h}{|\mathcal{Y}| - 1}d \\
        \vdots & \vdots &  & \vdots \\
        \gamma\dfrac{1-h}{|\mathcal{Y}| - 1}d & \gamma\dfrac{1-h}{|\mathcal{Y}| - 1}d & \cdots & \gamma hd + (1 - \gamma) \\
        \vdots & \vdots & \ddots & \vdots \\
        \gamma\dfrac{1-h}{|\mathcal{Y}| - 1}d & \gamma\dfrac{1-h}{|\mathcal{Y}| - 1}d & \cdots & \gamma hd + (1 - \gamma) \\
    \end{bmatrix}.
}
\end{equation}

Finally, we replace $A$ with $\Tilde{A} = D^{-1/2} A D^{-1/2}$. Since each node has the same degree $d$, each signal is now multiplied with $(1/\sqrt{d})^2 = 1/d$ during propagation, eventually cancelling out the $d$'s:
\begin{equation}
\Scale[0.8]{
    (\gamma \Tilde{A} + (1 - \gamma)I)Y =
    \begin{bmatrix}
        \gamma h + (1 - \gamma) &  \gamma\dfrac{1-h}{|\mathcal{Y}| - 1} & \cdots & \gamma\dfrac{1-h}{|\mathcal{Y}| - 1} \\
        \vdots & \vdots & \ddots & \vdots \\
        \gamma h + (1 - \gamma) & \gamma\dfrac{1-h}{|\mathcal{Y}| - 1} & \cdots & \gamma\dfrac{1-h}{|\mathcal{Y}| - 1} \\
        \gamma\dfrac{1-h}{|\mathcal{Y}| - 1} & \gamma h + (1 - \gamma) & \cdots & \gamma\dfrac{1-h}{|\mathcal{Y}| - 1} \\
        \vdots & \vdots & \ddots & \vdots \\
        \gamma\dfrac{1-h}{|\mathcal{Y}| - 1} & \gamma h + (1 - \gamma) & \cdots & \gamma \dfrac{1-h}{|\mathcal{Y}| - 1} \\
        \vdots & \vdots &  & \vdots \\
        \gamma\dfrac{1-h}{|\mathcal{Y}| - 1} & \gamma\dfrac{1-h}{|\mathcal{Y}| - 1} & \cdots & \gamma h + (1 - \gamma) \\
        \vdots & \vdots & \ddots & \vdots \\
        \gamma\dfrac{1-h}{|\mathcal{Y}| - 1} & \gamma\dfrac{1-h}{|\mathcal{Y}| - 1} & \cdots & \gamma h + (1 - \gamma) \\
    \end{bmatrix}.
}
\end{equation}

Now, we are ready to change $Y$ to a matrix of probability vectors ({\it i.e.}, $P^t$), which finally results in calculating one iteration of $\Bar{\Pi}$. In our analysis, we replace $Y$ with a $P^t \in [0,1]^{|V| \times |\mathcal{Y}|}$, where the correct label is predicted with probability $p$ and the rest of the probabilities $(1-p)$ is distributed uniformly for the rest of the classes:
\begin{equation}
\Scale[0.8]{
    P^t =
    \begin{bmatrix}
        p & \dfrac{1-p}{|\mathcal{Y}| - 1} & \cdots & \dfrac{1-p}{|\mathcal{Y}| - 1} \\
        \vdots & \vdots & \ddots & \vdots \\
        p & \dfrac{1-p}{|\mathcal{Y}| - 1} & \cdots & \dfrac{1-p}{|\mathcal{Y}| - 1} \\
        \dfrac{1-p}{|\mathcal{Y}| - 1} & p & \cdots & \dfrac{1-p}{|\mathcal{Y}| - 1} \\
        \vdots & \vdots & \ddots & \vdots \\
        \dfrac{1-p}{|\mathcal{Y}| - 1} & p & \cdots & \dfrac{1-p}{|\mathcal{Y}| - 1} \\
        \vdots & \vdots &  & \vdots \\
        \dfrac{1-p}{|\mathcal{Y}| - 1} & \dfrac{1-p}{|\mathcal{Y}| - 1} & \cdots & p \\
        \vdots & \vdots & \ddots & \vdots \\
        \dfrac{1-p}{|\mathcal{Y}| - 1} & \dfrac{1-p}{|\mathcal{Y}| - 1} & \cdots & p \\
    \end{bmatrix}.
}
\end{equation}
Assuming $1/|\mathcal{Y}| < p \leq 1$ implies that the teacher GNN has the accuracy of 1 ({\it i.e.}, perfect prediction). Note that setting $p = 1$ reverts $P^t$ to $Y$. Now, calculating for $(\gamma \Tilde{A} + (1 - \gamma)I)P^t$ results in:
\begin{equation*}
\Scale[0.8]{(\gamma \Tilde{A} + (1 - \gamma)I)P^t=
    \begin{bmatrix}
        \beta &  \beta' & \cdots & \beta' \\
        \vdots & \vdots & \ddots & \vdots \\
        \beta & \beta' & \cdots & \beta' \\
        \beta' & \beta & \cdots & \beta' \\
        \vdots & \vdots & \ddots & \vdots \\
        \beta' & \beta & \cdots & \beta' \\
        \vdots & \vdots &  & \vdots \\
        \beta' & \beta' & \cdots & \beta \\
        \vdots & \vdots & \ddots & \vdots \\
        \beta' & \beta' & \cdots & \beta \\
    \end{bmatrix}
    },
\end{equation*}
where
\begin{align}
    \beta &= (1 - \gamma)p + \gamma hp + \gamma\dfrac{1-h}{|\mathcal{Y}| - 1}(1-p) \label{eq:beta_q}\\
    \beta' &= (1 - \gamma)\dfrac{1-p}{|\mathcal{Y}| - 1} + \gamma\dfrac{h}{|\mathcal{Y}| - 1}(1-p) + \gamma\dfrac{1-h}{|\mathcal{Y}| - 1}p + \gamma(1-h)\dfrac{|\mathcal{Y}| - 2}{(|\mathcal{Y}| - 1)^2}(1-p)\label{eq:beta'_q}.
\end{align}


Now, we are ready to prove Theorem~\ref{theorem:smoothingcorrectionwitherror}.

In order to prove the theorem, we basically follow the same calculation steps by recalculating the interval for $v_* = v_1$ while including $\epsilon$, which will require more careful considerations. We start with revisiting Eqs.~(\ref{eq:beta_q}) and (\ref{eq:beta'_q}). In Eq.~(\ref{eq:beta_q}), there are three terms, {\it i.e.}, $\gamma hp$, $(1 - \gamma)q$, and $\gamma(1-p)\dfrac{1-h}{|\mathcal{Y}|-1}$. For ease of notations, we denote the set of nodes with the same labels with $v_1$ as $S$ and the rest of neighbors as $S'$ in the previous setting where we assumed that nodes other than $v_1$ were all correct.

Starting with $(1 - \gamma)q$, this term calculates the effect of self-propagation and therefore remains unchanged in the new setting. The term $\gamma hp$ calculates the influence that is aggregated from nodes in $S$. In the new setting, only $(1-\epsilon)hd$ nodes propagate the probability $p$ (Note that $|S| = hd$), and therefore $\gamma hp$ is changed into $\gamma h\left((1-\epsilon)p + \epsilon \dfrac{1 - p}{|\mathcal{Y}| - 1}\right)$. Next, the term $\gamma(1-h)\dfrac{1-p}{|\mathcal{Y}|-1}$ calculates the influence aggregated from nodes in $S'$, which previously all propagated $\dfrac{1-p}{|\mathcal{Y}|-1}$. When $\epsilon=0$, the number of these nodes is $|S'|=(1-h)d$, and in the new setting, $\epsilon (1-h)d$ has their predictions changed, where $\dfrac{1}{|\mathcal{Y}|-1}$ of them now propagates $p$. Therefore, this term is now changed into $\gamma (1-h)\left( \epsilon \dfrac{1}{|\mathcal{Y}|-1} p + \dfrac{|\mathcal{Y}|-1 - \epsilon}{(|\mathcal{Y}|-1)^2}(1-p)\right)$. In summary, in the new setting, $\beta_q$ becomes 

\begin{align}
    \beta_{q, \epsilon} &= (1-\gamma)q + \gamma h\left((1-\epsilon)p + \epsilon \dfrac{1 - p}{|\mathcal{Y}| - 1}\right) \notag \\
    &+ \gamma (1-h)\left( \epsilon \dfrac{1}{|\mathcal{Y}|-1} p + \dfrac{|\mathcal{Y}|-1 - \epsilon}{(|\mathcal{Y}|-1)^2}(1-p)\right).
\end{align}

In Eq.~(\ref{eq:beta'_q}), there are four terms, {\it i.e.}, $(1 - \gamma)\dfrac{1-q}{|\mathcal{Y}| - 1}$, $\gamma\dfrac{h}{|\mathcal{Y}| - 1}(1-p)$, $\gamma\dfrac{1-h}{|\mathcal{Y}| - 1}p$, and $\gamma(1-h)\dfrac{|\mathcal{Y}| - 2}{(|\mathcal{Y}| - 1)^2}(1-p)$. Using the assumption that the predictions are uniformly distributed among classes for nodes in $S'$; without loss of generality, let us calculate the probability regarding the second class label.

Similarly as in $\beta_q$, the term $(1 - \gamma)\dfrac{1-q}{|\mathcal{Y}| - 1}$ remains unaffected in the new setting as it is the result of self-propagation. The term $\gamma\dfrac{h}{|\mathcal{Y}| - 1}(1-p)$ calculates the influence from nodes that were in $S$. For these $|S| = hd$ nodes, they previously propagated $\dfrac{1-p}{|\mathcal{Y}| - 1}$. In the new setting, $\epsilon \dfrac{1}{|\mathcal{Y}| - 1} hd $ nodes now predict the second class and propagate $p$, while the rest of the $\dfrac{|\mathcal{Y}| - 1 - \epsilon}{|\mathcal{Y}| - 1}hd$ nodes still propagates $\dfrac{1-p}{|\mathcal{Y}| - 1}$. In total, this term is now modified into $\gamma h \left(\dfrac{\epsilon}{|\mathcal{Y}| - 1}p + \dfrac{|\mathcal{Y}| - 1 - \epsilon}{(|\mathcal{Y}| - 1)^2}(1-p)\right)$. The term $\gamma\dfrac{1-h}{|\mathcal{Y}| - 1}p$ calculates the influence from nodes that previously predicted the second class in $S'$. The nodes are now split into two groups with ratio $\epsilon : (1 - \epsilon)$, where the first group now propagates $\dfrac{1-p}{|\mathcal{Y}| - 1}$ and the latter still propagates $p$. In total, this term is now modified into $\gamma \dfrac{1-h}{|\mathcal{Y}| - 1}\left(\epsilon \dfrac{1}{|\mathcal{Y}| - 1}(1-p) + (1-\epsilon)p\right)$. Next, the term $\gamma(1-h)(1-p)\dfrac{|\mathcal{Y}| - 2}{(|\mathcal{Y}| - 1)^2}$ calculates the influence from nodes that previously did not predict as the second class in $S'$. Similarly as before, the nodes are now split into two groups with ratio $\dfrac{1}{|\mathcal{Y}| - 1}\epsilon : \dfrac{|\mathcal{Y}| - 1 - \epsilon}{|\mathcal{Y}| - 1}$, where the first group now (incorrectly) predicts the second class and thus propagates $p$, while the latter still propagates $\dfrac{1-p}{|\mathcal{Y}| - 1}$. In total, this term is now modified into $\gamma(1-h)\dfrac{|\mathcal{Y}| - 2}{|\mathcal{Y}| - 1}\left(\dfrac{\epsilon}{|\mathcal{Y}| - 1}p + \dfrac{|\mathcal{Y}| - 1 - \epsilon}{(|\mathcal{Y}| - 1)^2}(1-p)\right)$. In summary, in the new setting, $\beta'_q$ becomes 
\begin{align}
    \beta'_{q, \epsilon} &= (1 - \gamma)\dfrac{1-q}{|\mathcal{Y}| - 1} + \gamma h \left(\dfrac{\epsilon}{|\mathcal{Y}| - 1}p + \dfrac{|\mathcal{Y}| - 1 - \epsilon}{(|\mathcal{Y}| - 1)^2}(1-p)\right) \nonumber \\ 
    &+ \gamma \dfrac{1-h}{|\mathcal{Y}| - 1}\left(\epsilon \dfrac{1}{|\mathcal{Y}| - 1}(1-p) + (1-\epsilon)p\right) \nonumber\\ 
    &+ \gamma(1-h)\dfrac{|\mathcal{Y}| - 2}{|\mathcal{Y}| - 1}\left(\dfrac{\epsilon}{|\mathcal{Y}| - 1}p + \dfrac{|\mathcal{Y}| - 1 - \epsilon}{(|\mathcal{Y}| - 1)^2}(1-p)\right).
\end{align}

We can verify that both $\beta_{q, \epsilon = 0}$ and $\beta'_{q, \epsilon = 0}$ reduce to $\beta_{q}$ and $\beta'_q$, respectively. Now, in the scenario where the node prediction is corrected after propagation, we need $\beta_{q, \epsilon} > \beta'_{q, \epsilon}$. After calculation with similar approximations when we calculated Eq.~(\ref{equation:qwithgoodteacher}), we arrive at:
\begin{align}
    q &> \dfrac{1}{|\mathcal{Y}|} - \dfrac{\gamma}{1 - \gamma}\left(\left((1 - \epsilon) + \dfrac{1 - 2\epsilon}{|\mathcal{Y}|}\right)hp - (1-\epsilon)\dfrac{h + p}{|\mathcal{Y}|}\right) \notag\\ 
    &= \dfrac{1}{|\mathcal{Y}|} - \dfrac{\gamma}{1 - \gamma}\left(C - \epsilon\left(C + \dfrac{hp}{|\mathcal{Y}|}\right)\right),
\end{align}
which concludes the proof of Theorem~\ref{theorem:smoothingcorrectionwitherror}.
\end{proof}

From Theorem~\ref{theorem:smoothingcorrectionwitherror}, we can see that an increase of the error $\epsilon$ reduces the range of $q$, enabling the prediction of nodes to get corrected, which means that incorrect predictions from the teacher GNN introduce a more unforgiving environment for self-correction. Moreover, it is worth noting that the acceptable amount of error such that corrections via propagation are possible is upper-bounded by
\begin{equation}
    \epsilon < \frac{|\mathcal{Y}|h - 1}{(|\mathcal{Y}| + 1)h - 1},
\end{equation}
which monotonically increases with $h \in (1/|\mathcal{Y}|, 1]$. This implies that stronger homophily of the underlying graph will result in more tolerance of the prediction error from the teacher GNN.

\subsection{Analysis for $\epsilon=0$}

Let us also consider a simpler scenario where the teacher GNN makes an incorrect prediction \textit{only} for a single node $v_*$ ({\it i.e.}, $\epsilon = 0$) with class 0 as its ground truth label by assigning a new output probability vector $P^t[*,:] = [q, \dfrac{1-q}{|\mathcal{Y} - 1|}, \cdots, \dfrac{1-q}{|\mathcal{Y} - 1|}]$ ($0 \leq q < 1/|\mathcal{Y}|$), same as in our previous setting in Theorem~\ref{theorem:smoothingcorrectionwitherror}. In this setting, we establish the following corollary.

\begin{corollary}
\label{corollary:appendix}
Suppose that the teacher make an incorrect prediction only for a single node $v_*$. Using one iteration of propagation in Eq.~(\ref{equation:labelpropagation}), the prediction of node $v_*$ gets corrected if
    \begin{equation}
        q \in \left[\max \left(0, \dfrac{1}{|\mathcal{Y}|} - \dfrac{\gamma}{1 - \gamma} C \right), \dfrac{1}{|\mathcal{Y}|}\right],
    \end{equation}
where $q$ is the output probability of $v_*$ corresponding to class 0 and $C$ is approximately $\left(1 + \dfrac{1}{|\mathcal{Y}|}\right)hp - \dfrac{h+p}{|\mathcal{Y}|}$.
\end{corollary}

\begin{proof}
Since we consider the case where $\epsilon=0$ ({\it i.e.}, the teacher GNN provides correct predictions for nodes other than $v^{*}$), the resulting vector for the first row of $(\gamma \Tilde{A} + (1 - \gamma)I)P^t$ can be expressed as:
\begin{equation}
    ((\gamma \Tilde{A} + (1 - \gamma)I)P^t)[1,:] = [\beta_q, \underbrace{\beta_q', \cdots, \beta_q'}_{(|\mathcal{Y}| - 1)}].
\end{equation}

Intuitively, $\beta_q$ represents the result after propagation for the correct class, and $\beta'_q$ represents the rest of the (incorrect) classes. For the incorrect prediction to be corrected after propagation, it requires $\beta_q >\beta'_q$:
\begin{align}
    (1 - \gamma)\left(q - \dfrac{1-q}{|\mathcal{Y}|-1}\right) &> -\gamma hp - \gamma \dfrac{1-h}{|\mathcal{Y}|-1}(1-p) + \gamma\dfrac{h}{|\mathcal{Y}| - 1}(1-p) \nonumber \\
    &+ \gamma\dfrac{1-h}{|\mathcal{Y}| - 1}p + \gamma(1-h)\dfrac{|\mathcal{Y}| - 2}{(|\mathcal{Y}| - 1)^2}(1-p). \label{eq:intermediate}
\end{align}
Calculation of Eq.~(\ref{eq:intermediate}) can directly reveal the condition for the correction scenario. We can directly derive an interval for $q$ by approximating $\dfrac{|\mathcal{Y}|-2}{|\mathcal{Y}|-1} \approx 1$, which further reduces Eq.~(\ref{eq:intermediate}) to
\begin{equation}
\label{equation:qwithgoodteacher}
    q > \dfrac{1}{|\mathcal{Y}|} - \dfrac{\gamma}{1 - \gamma}\left(\left(1 + \dfrac{1}{|\mathcal{Y}|}\right)hp - \dfrac{h + p}{|\mathcal{Y}|}\right).
\end{equation}

Denoting $C = \left(1 + \dfrac{1}{|\mathcal{Y}|}\right)hp - \dfrac{h+p}{|\mathcal{Y}|}$ results in the interval Eq.~(\ref{eq:intermediate}), which concludes the proof of Corollary~\ref{corollary:appendix}.
\end{proof}

\section{Experimental Results on a Larger Dataset}
We perform an additional experiment on the large-scale OGB-arxiv dataset, where the number of nodes, number of edges, number of node features, and number of classes are 169,343, 1,166,243, 128, and 40, respectively. The experiment is performed in the transductive scenario for \textbf{P\&D} and \textbf{P\&D}-fix, since InvKD is computationally impractical to perform on the large dataset.

\begin{table}[h]
\caption{Node classification accuracy (\%) for OGB-arxiv in transductive setting. The columns represent the performance of the teacher GNN model, plain MLP model without distillation, GLNN~\citep{zhang2022glnn}, and two versions of {\bf P\&D}. The performance of the best method is denoted in bold font.}
\label{table:arxivexperiment}
\small
\centering
\begin{tabular}{lcccccc}
\toprule
Dataset  & Teacher GNN & Plain MLP & GLNN & \textbf{P\&D} & \textbf{P\&D}-fix \\
\midrule
OGB-arxiv & 70.64 ± 0.41 & 55.33 ± 1.54 & 63.02 ± 0.41 & \textbf{65.20} ± 0.45 & 65.14 ± 0.35 \\
\bottomrule
\end{tabular}
\end{table}

As shown in Table~\ref{table:arxivexperiment}, experimental results demonstrate that a similar trend is also found for the OGB-arxiv dataset, where \textbf{P\&}D and \textbf{P\&D}-fix achieve gains of 2.18\% and 2.12\% compared to GLNN, respectively.

\section{Experimental Results on an Alternative Teacher GNN}
We perform an additional set of experiments in the transductive setting where APPNP is adopted as the teacher GNN model in addition to GraphSAGE. 

\begin{table}[h]
\caption{Node classification accuracy (\%) for Cora, Citeseer, Pubmed in the transductive setting employing GraphSAGE and APPNP as the teacher GNN. The columns represent the model used for another teacher GNN, datasets, the performance of GLNN~\citep{zhang2022glnn}, InvKD, and two versions of {\bf P\&D}, alongside the performance increase with respect to GLNN in the parenthesis. The case with higher performance increase is denoted as bold font.}
\label{table:appnpexperiment}
\small
\centering
\begin{tabular}{llcccc}
\toprule
Teacher GNN & Dataset  & GLNN & InvKD & \textbf{P\&D} & \textbf{P\&D}-fix \\
\midrule
& Cora & 80.73 & 82.22 (↑\textbf{1.49}) & 82.16 (↑\textbf{1.43}) & 82.29 (↑\textbf{1.56}) \\
SAGE & Citeseer & 71.19 & 74.08 (↑\textbf{2.89}) & 73.38 (↑\textbf{2.19}) & 74.93 (↑\textbf{3.74}) \\
& Pubmed & 76.39 & 77.22 (↑0.83) & 77.88 (↑\textbf{1.49}) & 78.11 (↑\textbf{1.72}) \\
\midrule
& Cora & 82.81 & 83.27 (↑0.46) & 83.35 (↑0.54) & 83.82 (↑1.01) \\
APPNP & Citeseer & 73.02 & 73.55 (↑0.53) & 74.32 (↑0.30) & 74.18 (↑1.16) \\
& Pubmed & 75.92 & 77.24 (↑\textbf{1.32}) & 77.37 (↑1.45) & 77.60 (↑1.68) \\
\bottomrule
\end{tabular}
\end{table}

Table~\ref{table:appnpexperiment} summarizes the experimental results, including the relative performance gain over GLNN between two cases: using GraphSAGE and APPNP as the teacher GNNs. 

From the experimental results, we would like to make the following observations.
\begin{itemize}
    \item The proposed framework still outperforms GLNN when APPNP is used as the teacher GNN.
    \item The gain over GLNN tends to be decreased in comparison with the case where GraphSAGE is used as the teacher GNN. This is because a significant portion of the additional structural information that our framework injects is already learned by APPNP due to the similarity in terms of propagation, which potentially makes GLNN more potent.
\end{itemize}
We expect that developing a more sophisticated propagation method that can combine structural information that cannot be captured by propagation will alleviate this phenomenon, which we aim to tackle in our future work.

\end{document}